\documentclass{article}
\usepackage{ijcai19}

\usepackage{times}
\usepackage{soul}
\usepackage{url}
\usepackage[hidelinks]{hyperref}
\usepackage[utf8]{inputenc}
\usepackage[small]{caption}
\usepackage{graphicx}
\usepackage{amsmath}
\usepackage{booktabs}
\usepackage{algorithm}
\usepackage{algorithmic}
\usepackage{xpatch}
\usepackage{enumitem}
\usepackage{amsthm}
\usepackage{amssymb}
\newtheorem{theorem}{Theorem}[section]

\newtheorem{lemma}[theorem]{Lemma}
\makeatletter
\xpatchcmd{\proof}{\topsep5\p@\@plus6\p@\relax}{}{}{}
\makeatother

\title{Graph Space Embedding}

\author{
Jo\~{a}o Pereira$^{1,2}$
\and
Albert K. Groen$^1$\and
 Erik S. G. Stroes$^1$\And
Evgeni Levin$^{1,2}$
\affiliations
$^1$Amsterdam University Medical Center, The Netherlands\\
$^2$Horaizon BV, The Netherlands
\emails
\{j.p.belopereira, e.stroes, a.k.groen, e.levin\}@amsterdamumc.nl
}
\begin{document}

\maketitle

\begin{abstract}
We propose the \textit{Graph Space Embedding} (GSE), a technique that maps the input into a space where interactions are implicitly encoded, with little computations required. We provide theoretical results on an optimal regime for the GSE, namely a feasibility region for its parameters, and demonstrate the experimental relevance of our findings.
Next, we introduce a strategy to gain insight on which interactions are responsible for the certain predictions, paving the way for a far more transparent model. In an empirical evaluation on a real-world clinical cohort containing patients with suspected coronary artery disease, the GSE achieves far better performance than traditional algorithms.
\end{abstract}

\section{Introduction}

Learning from interconnected systems can be a particularly difficult task due to the possibly non-linear interaction between the components \cite{interactions,physics_interactions}.
%	Adequate risk prediction is considered a cornerstone in the management of patients with suspected coronary artery disease (CAD). However, traditional risk stratification using generally available clinical risk factors, plasma lipid levels and other conventional biomarkers have only modest predictive value for the occurrence of events and the presence of coronary atherosclerosis\cite{risk-prediction, risk-prediction2}. Recently, technical advances have enabled the simultaneous measurement of large amounts of proteins using a single drop of blood, and several inflammatory plasma proteins have already been linked to the pathophysiology of atherosclerosis since\cite{pea, inflammation}. However, these variables form a network whose interactions alter their individual function in a non-linear fashion. 
In some cases, these interactions are known and therefore constitute an important source of prior information \cite{robots,ppi_prior}. Although prior knowledge can be leveraged in a variety of ways \cite{domain_knowledge}, most of the research involving interactions, is focused on their discovery. One popular approach to deal with feature interactions, is to cast the interaction network as a graph and then use kernel methods based on graph properties, such as walk-lengths or subgraphs ~\cite{short_path_kernel,graphlets,subgraph} or, more recently, graph deep convolutional methods ~\cite{deep_graph,graph_cnn,kipf}. In this work however, we focus on the case in which the interactions are feature specific and a universal property of the data instances, which make the pattern search algorithms not suitable for this task. To our knowledge, there is limited research involving this setting, although we suggest many problems can be formulated in the same way (see Figure ~\ref{fig:approach}).  %Such prior knowledge can be used to prepare the data, e.g. transforming the features prior to the analysis; initiate the hypothesis, e.g. using domain knowledge to initiate Bayesian networks' structure; changing the search objective; or even to restrict the learning space \cite{domain_knowledge}. 
 To address this knowledge gap, we present a novel method: \textit{Graph Space Embedding} (GSE), an approach related to the 'random-walk' graph kernel \cite{gartner,fast_rwk} with an important difference: it is not limited to the sum of all walks of a given length, but rather compares similar edges in two different graphs,  which results in better expressiveness. Our empirical evaluation demonstrates that GSE leads to an improvement in performance compared to other baseline algorithms when plasma protein measurements and their interactions are used to predict ischaemia in patients with Coronary Artery Disease (CAD) \cite{cad,cad4}. Moreover, the kernel can be computed in $\mathcal{O}(n^2)$, where $n$ is the number of features, and its hyperparameters efficiently optimized via maximization of the kernel matrix variation.
\subsection{Main Contributions}
\begin{enumerate}
\item \textit{Graph Space Embedding} function that efficiently maps input into an “interaction-based” space
\item Novel theoretical result on optimal regime for the GSE, namely feasibility region for its parameters
\item \textit{Even Decent Sampling Algorithm}: a strategy to gain insight on which interactions are responsible for the certain prediction
\end{enumerate}
\begin{figure}[tp!]\
	\centering
	\includegraphics[width=0.6\linewidth]{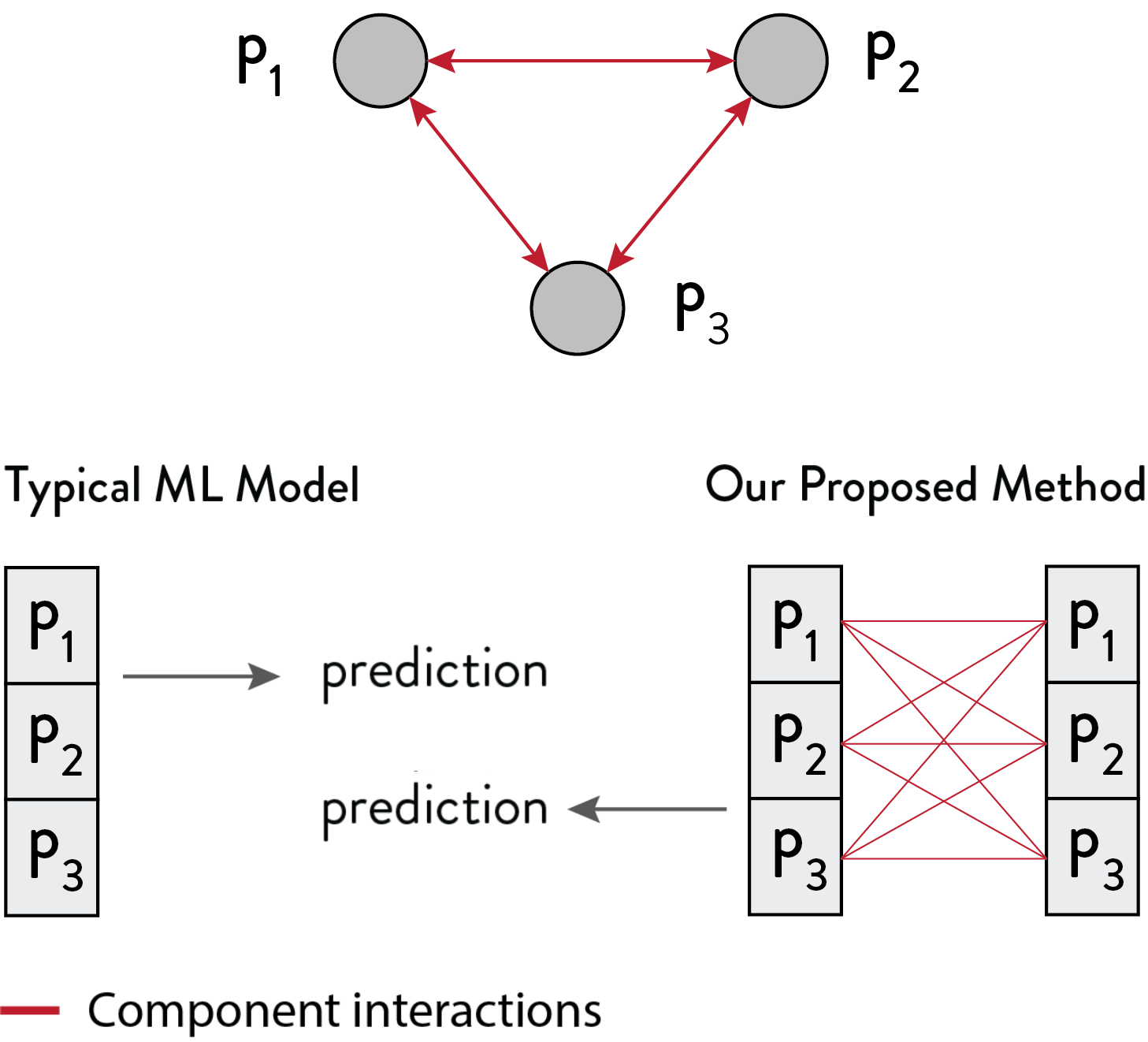}
	\caption[Standard learning algorithm vs PPI]{A traditional learning algorithm with no structural information will take the feature values and learn to produce a prediction with complete disregard for their interactions (top graph).}
	\label{fig:approach}
\end{figure}

\section{Approach}
A remark on notation: we will use bold capital letters for matrices, bold letters for arrays and lower case letters for scalars/functions/1-d variables  (ex. $\bold{X}, \bold{x}, x$).
\subsection{Interaction Graphs}

Any network can be represented by a graph $\mathcal{G}=\{V,E\}$, where $E$ is a set of edges, $V$ a set of vertices. Denote by $\bold{A}_{|V| \times |V|}$ ($|V|$ is equal to the number of features $N$) the adjacency matrix, where $\bold{A}_{i,j}$ represents the interaction between feature $i$ and $j$, and whose value is $0$ if there is no interaction. \par
%For simplicity, we will assume that each variable affects each other reciprocally, which amounts to setting $\bold{A}_{i,j} = \bold{A}_{j,i}$, although this is not a hard constraint. \\
Let $\bold{x}_{_{1\times N}}$ be an array with measurements of features $1$ to $N$ for a given point in the data. In order to construct an instance-specific matrix, one can weigh the interaction between each pair of features with a function of their values' product:
\begin{equation}\label{eq:G}
\bold{G}_{\bold{x}}(\bold{A}) = \varphi (\bold{A}) \circ \bold{x}^{\top} \bold{x},
\end{equation}
where $\varphi(\bold{A})$ is some function of the network interaction matrix $\bold{A}$, and the operator $\circ$ represents the Hadamard product, i.e. $(\bold{A} \circ \bold{B} )_{i,j} = (\bold{A})_{i,j} (\bold{B} )_{i,j}$.

\subsection{Graph Kernel}

Unlike the distance in euclidean geometry, which intuitively represents the length of a line between two points, there is no such tangible metric for graphs. Instead, one has to decide what is a reasonable evaluation for the difference between two graphs in the context of the problem. \par A popular approach \cite{gartner} is to compare random walks on both graphs.
The $i,j$th entry of the order $k$ power of an adjacency matrix $\bold{A}_{_{ |V|\times |V|}}$: $\bold{A}^k = \underbrace{\bold{A}\bold{A}...\bold{A}}_{k\,times}$, corresponds to the number of walks of length $k$ from $i$ to $j$. Any function that maps the data into a feature space $\mathcal{H}$: $\phi: X \rightarrow \mathcal{H}$, $k(\bold{x},\bold{y}) = <\phi(\bold{x}), \phi(\bold{y})>$ is a kernel function. Using the original graph kernel formulation, it is possible to define a kernel that will implicitly map the data into a space where the interactions are incorporated:
\begin{equation}\label{original_gartner}
k_{n}(\bold{G}, \bold{G}^{'}) = \sum^{n}_{i,j=1}[ \gamma]_{i,j} \left \langle  [\bold{G}]^{i} ,  [\bold{G}']^{j} \right \rangle_{F} ,
\end{equation}
where $\bold{G}$  and $\bold{G}'$ correspond to $\bold{G}_\bold{x}(\bold{A})$ and $\bold{G}_{\bold{x}'}(\bold{A})$ (see eq.~\ref{eq:G}); $\gamma_{i,j}$ is a function that "controls" the mapping $\phi(\cdot)$; and $n$ is the maximum allowed "random walks" length. If $\gamma$ is decomposed into $\bold{U}\bold{\Lambda} \bold{U}^T$, where $\bold{U}$ is a matrix whose columns are the eigenvectors of $\gamma$, and $\bold{\Lambda}$ a diagonal matrix with its eigenvalues at each diagonal entry, then equation ~\ref{original_gartner} can be re-factored into:
\begin{equation}
k_{n}(\bold{G}, \bold{G}^{'}) = \sum_{k,l=1}^{|V|} \sum^{n}_{i=1} \phi_{i,k,l}(\bold{G})\phi_{i,k,l}(\bold{G}'),
\end{equation}
where $\phi_{i,k,l}(\bold{G}) = \sum_{j=1}^n [\sqrt{\bold{\Lambda}} \bold{U}^T]_{i,j} \bold{G}^j$. Consequently, different forms of the function $\gamma$ can be chosen, with different interpretations. For the case where $\gamma_{i,j} = \theta^i\theta^j$, which yields:
\begin{equation}
\begin{aligned}
k_n(\bold{G},\bold{G}') = \langle  \sum_{i=1}^n\theta^i [\bold{G}]^i ,  \sum_{j=1}^n\theta^j [\bold{G}']^j  \rangle_F
\\ = \langle  \sum_{i=1}^n\theta^i [\bold{G}]^i ,  \sum_{i=1}^n\theta^i [\bold{G}']^i  \rangle_F,
\end{aligned}
\end{equation}
the kernel entry can be interpreted as an inner product in a space where there is a feature for every node pair \{$ k, l$\}, which represents the weighted sum of paths of length $1$ to $n$ from $k$ to $l$ $(\phi_{k,l} = \sum_{i=1}^n \theta^i \bold{G}_{k,l}^i)$ \cite{Evgeni}. The kernel can then be used with a method that employs the kernel trick, such as support vector machines,
%\cite{vapnik}
 kernel PCA
% \cite{kpca,kpca2}
 or kernel clustering.
%  \cite{kknn,kclustering}.
Another interesting case is when we consider the weighted sum of paths of length $1$ to $\infty$. This can be calculated using:
\begin{equation}\label{eq:exp_ip}
k_{\infty}(\bold{G}, \bold{G}^{'}) =\langle  e ^ {\beta \bold{G}} ,  e ^ {\beta \bold{G}'}  \rangle_F,
\end{equation}
since $e ^ {\beta \bold{G}} = \lim_{ n \to +\infty} \sum_{i=0}^n \frac{\beta^i}{i!} \bold{G}^i $, where $\beta$ is a parameter.

\subsection{Graph Space Embedding}

Since we are dealing with a universal interaction matrix for every data point and the interactions are feature specific, it makes sense to compare the same set of edges for every pair of points. As a consequence, we can also avoid solving time-consuming graph structure problems. With these two points in mind, we combined the previous graph kernel methods and the radial basis function (RBF) %refer to smola book
to develop a new kernel which we will henceforth refer to as Graph Space Embedding (GSE).
The radial basis function is defined as:
\begin{equation}\label{rbf}
k(\bold{x}, \bold{y}) = e^{-\frac{||\bold{x}-\bold{y}||^2}{\sigma ^2}} = c \, e ^{ \frac{2<\bold{x}, \bold{y}>}{\sigma^2}},
\end{equation}
where $c=e ^{ -\frac{||\bold{x}||^2}{\sigma^2}}e ^{ -\frac{||\bold{y}||^2}{\sigma^2}}$ . 
The GSE uses the distance $\left \langle  \sqrt{\gamma}[\bold{G}],  \sqrt{\gamma}[\bold{G}'] \right \rangle_{F} $ in the radial basis function:
\begin{equation}\label{eq:rbf_w_g}
\resizebox{.91\linewidth}{!}{$k(\bold{G}, \bold{G}')= c \, e ^{ \frac{2<\bold{x}, \bold{y}>}{\sigma^2}} = 
c \, \underbrace{\sum_{n=0}^{\infty}  \frac{ \left ( 2\left<\sqrt{\gamma} \, \bold{G}, \sqrt{\gamma} \, \bold{G}'\right>_{F} \right )^n}{\sigma^{2n}\,n!}}_{r\_w}$}
\end{equation}
\par If we then take the upper term of the fraction in $r\_w$ to be $\left[2\sum_{i = 0 }^{|E|}\gamma \, \bold{G}_{i}\bold{G}'_{i}\right]^n$, we can use the multinomial theorem to expand each term of the exponential power series, and the expression for the kernel then becomes:

\begin{equation}\label{final_rbf_ppi}
k(\bold{G}, \bold{G}^{'})=c \, \sum_{n=0}^{\infty}\underbrace{ \left(\frac{2}{\nu}\right) ^n}_{\lambda} \underbrace{\sum_{\boldsymbol{\alpha}^n(\cdot)}
	\frac{\prod_{i=1}^{|E|}[\,\bold{G}_i \bold{G}'_i]^{\alpha_i}}{\prod_{i=1}^{|E|}\Gamma(\alpha_i + 1)}}_{r\_e} ,
\end{equation}
where $\Gamma$ is the gamma function, $\bold{G}_i\in E$ is the value of edge i in $\bold{G}$ and $\nu = \frac{\sigma^2}{\gamma}$. Here, $\boldsymbol{\alpha}^n(\cdot)$ represents a combination of $|E|$ integers: $(\alpha_1, \alpha_2, ... , \alpha_{|E|})$, with $\sum_i^{|E|} \boldsymbol{\alpha}_i^n(\cdot) = n$, and the sum in $r\_e$ is taken over all possible combinations of $\boldsymbol{\alpha}^n(\cdot)$.  For instance, for $n=3$ in a graph with $|E|=5$, possible examples of $\boldsymbol{\alpha}^3(\cdot)$ include $(0,1,1,1,0)$ or $(0, 2, 1, 0, 0)$ (see Figure ~\ref{RPSE_graph}).
\begin{figure}[htbp!]
	\vskip 0.1in
	\begin{center}
		\centerline{\includegraphics[width=1\columnwidth]{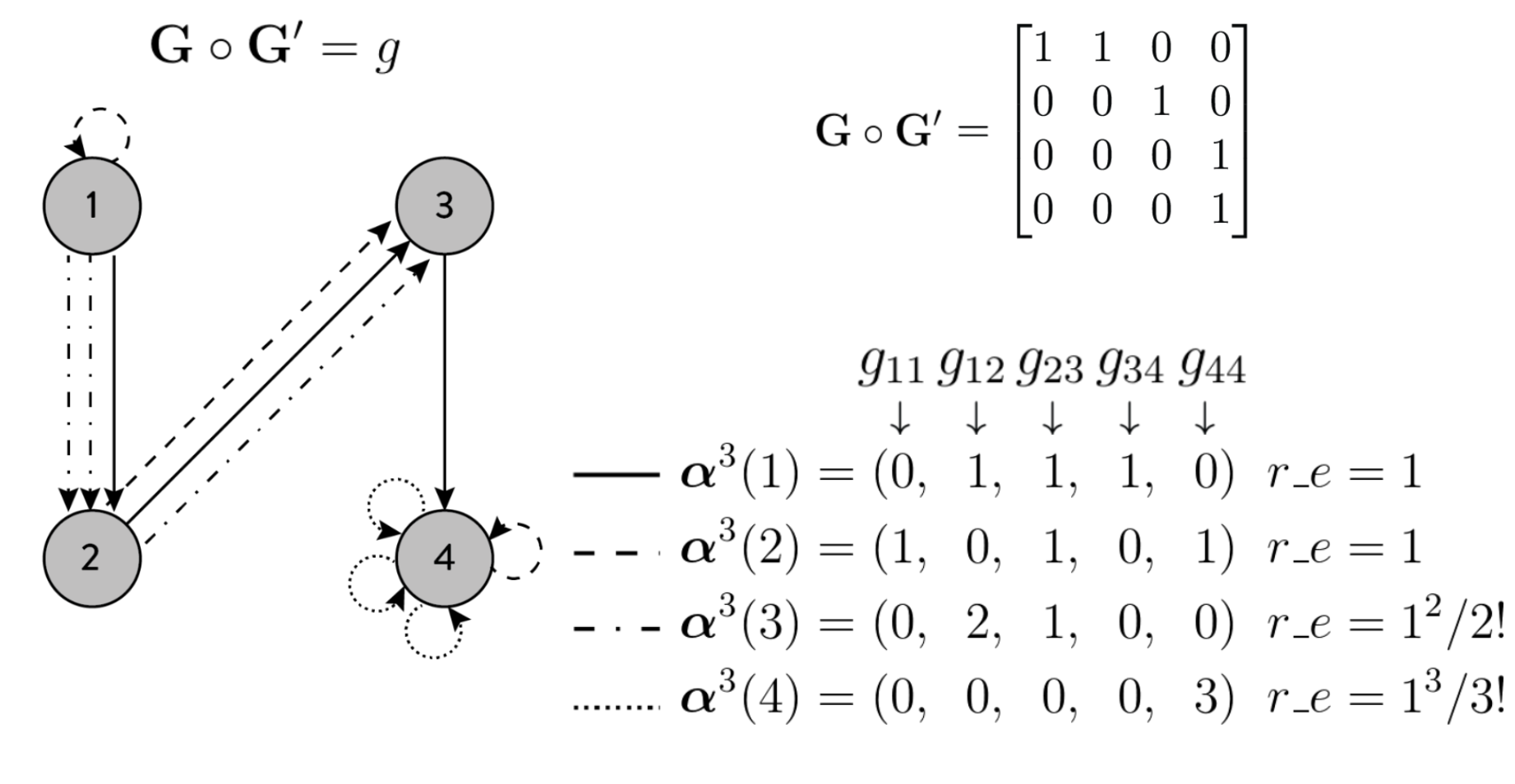}}
		\caption[Example of four walks of size 3 in an already multiplied graph]{The GSE kernel implicitly compares all edge combinations between $\bold{G}$ and $\bold{G}'$. In this hypothetical graph, we show a sample of four $\alpha$ combinations for $n=3$. We denote by $r\_e(\alpha(i))$ the value inside the sum $r\_e$ (see eq.~\ref{final_rbf_ppi}) corresponding to the combination $\alpha(i)$. Note that while $\alpha(1)$ is a graph walk and $\alpha(2)$ is not, $r\_e(\alpha(1))=r\_e(\alpha(2))$. However, due to the repetitions in $\alpha(3)$ and $\alpha(4)$, their value is shrunk in relation to the others. The higher the number of repetitions, the more the value shrinks.}
		\label{RPSE_graph}
	\end{center}
	\vskip -0.1in
\end{figure}
We begin by noting that since the sum in $r\_e$ is taken over all combinations  $(l,k) \in V\times V$ of size $n$, the GSE then represents a mapping from the input space to a space where all combinations of $n=0\rightarrow \infty$ edges are compared between $\bold{G}$ and $\bold{G}^{'}$, walks or otherwise (see fig \ref{RPSE_graph}).
Notice that this is in contrast with the kernel of equation \ref{eq:exp_ip}, where the comparison is between a sum of all possible walks of length $n=0\rightarrow \infty$ from one node to another in the two graphs. \par
The GSE also allows repeated edges. %commnet saying that is not walks that are compared but actually only the number of walks from one place to another, so little direct content is compareed.  confirm by expanding the kernel terms
However, if the data is normalized so that $\mu(\bold{G}_i)\simeq0, \sigma(\bold{G}_i)\simeq1$, then both the power in the numerator and the denominator of $r\_e$ will effectively dampen most combinations with repeated edges, with a higher dampening factor for higher number of repetitions and/or combinations. Even for outlier values, the gamma function will quickly dominate the numerator of $r\_e$.  The $\lambda$ factor serves the purpose of shrinking the combinations with higher number of edges for $\nu > 2$. Finally, $\sigma^2$ now serves a dual purpose: the usual one in RBF to control the influence of points in relation to their distance (see equation \ref{rbf}), while at the same time controlling how much combinations of increasing order are penalized.

\subsection{$\nu$ Feasibility Region}\label{var_region}

As discussed in the above section, the hyperparameter $\nu$ controls the shrinking of the contribution of higher order edge combinations. Intuitively, not all values of $\nu$ will yield a proper kernel matrix since too large of a value will leave out too many edge combinations while one too small will saturate the kernel values. This motivates the search for a $\nu$ value feasible operation region, where the kernel incorporates the necessary information for separability. 
Informally speaking, the kernel entry $k(\bold{G}, \bold{G}')$ measures the similarity of $\bold{G}$ and $\bold{G}'$. In case too few/many edge combinations are considered, the variation of the kernel values will be equal to $1$. Therefore, we use the variation of the kernel matrix $\sigma^2(\bold{K})$ as a proxy to detect if $\nu$ is within acceptable bounds. We shall refer to the ability of the kernel to map the points in the data into separable images $\phi(\bold{x})$ as kernel expressiveness.\par
To determine this region analytically, we find the $\nu_{max}$ that yields the largest kernel variation, and then use the loss function around this value to determine in which direction the value $\nu$ should take for minimal loss.
\begin{lemma}
$\max_{\nu}\,\,\sigma^2 \left(\bold{K}(\nu)\right)$ can be numerically estimated and is guaranteed to converge with a learning rate $\alpha \leq \frac{D}{2 (D-1)  d_{max}}$, where $D$ is the total number of inter graph combinations and $d_{max}$ is the largest combination distance.
\end{lemma}
\begin{proof}
The analytical expression for the variance is:
\begin{eqnarray}
\begin{aligned}
&\sigma^2 \left (\text{\footnotesize $\bold{K}(\nu)$} \right ) = E[\bold{K}(\nu) ^2] - \underbrace{E[\bold{K}(\nu)]^2}_{b}= \\ 
& \left( \frac{D-1}{D}\right) \sum_{d=1}^D e^{-2\nu d} -  \frac{1}{D^2} \sum_{i\neq j}^{D^2-D} 2e^{-\nu (d_i + d_j)}\,\,,
\end{aligned}
\end{eqnarray}
where we used the binomial theorem to expand $b$, and $d=||\bold{G}-\bold{G}'||^2$.
To guarantee the convergence of numerical methods the function derivative must be Lipschitz continuous:
\begin{equation}\label{eq: lip}
\frac{\| \bold{K}'(\nu) - \bold{K}'(\nu ')\|}{\|\nu - \nu '\|} \leq L(\bold{K}')\, : \forall \, \nu, \nu',
\end{equation}
by overloading the notation: $\bold{K}'(\nu) = \frac{\partial \sigma ^2 \left (\bold{K}(\nu)\right )}{\partial \nu}$ to simplify the expression. 
The left side of equation \ref{eq: lip} becomes:
\begin{equation}
\begin{split}
&\frac{\| \top - \Lambda \| }{\|\nu - \nu'\|}, \\
&\Lambda = 2\left ( \frac{D-1}{D^2} \right ) \left[ \sum_{d=1}^D d(e^{-2\nu d} - e^{-2\nu' d}) \right],\\
&\top = \frac{2}{D^2}\sum_{i\neq j}^{D^2-D} (d_i + d_j) \left (e^{-\nu (d_i + d_j)} - e^{-\nu' (d_i + d_j)}\right).
\end{split}
\end{equation}
Since $0\leq e^{-\beta} \leq 1 \, : \forall \, \beta \in  \mathbb{R}$, then:
\begin{equation}
\| \top - \Lambda \| \leq 2 \left ( \frac{D-1}{D} \right ) d_{max} .
\end{equation}
%\begin{eqnarray}
%|\Lambda | \leq 2 \left ( \frac{D-1}{D} \right ) d_{max}, \\
%|\top | \leq 4 \left( \frac{
%	D-1}{D} \right ) d_{max}, \\
%\| \top - \Lambda \| \leq 2 \left ( \frac{D-1}{D} \right ) d_{max} 
%\end{eqnarray}
 When $\epsilon=\nu - \nu'\rightarrow 0$ :
\begin{equation}
e^{-c\nu}-e^{-c\nu'} = \underbrace{\frac{e^{c\nu'}-e^{c\nu}}{e^{c(\nu + \nu')}}}_{\delta}\rightarrow 0, \, :  \, \nu, \nu' > 0,
\end{equation}
and $\delta$ tends much faster to 0 then $\epsilon$, since the denominator of $\delta$ is the exponential of the sum of $\nu$ and $\nu'$.
Thus, the function $k'(\nu)$ is Lipschitz continuous with constant equal to: $L(\bold{K}'(\nu))=2 \left ( \frac{D-1}{D} \right )d_{max}$ .
\end{proof}
\par We shall later demonstrate empirically that $\nu^* = \max_{\nu}\,\, \sigma^2(\bold{K}(\nu))$ improves the class separability for our dataset.

\subsection{Comparison with Standard Graph Kernels}\label{GK_theorectical_comparison}

The original formulation of the graph kernel by Gartner et. al (see eq. ~\ref{original_gartner}), multiplies sums of random walks of length $i$ from one edge to another ($k\rightarrow l$) by sums of random walks $k\rightarrow l$ from the other graph being compared of a length not necessarily equal to $i$:
\begin{equation}
\begin{split}
k_{n}(\bold{G}, \bold{G}^{'}) = \sum^{n}_{i,j=1}[ \gamma]_{i,j} \left \langle  [\bold{G}]^{i}_{kl} ,  [\bold{G}']^{j}_{kl} \right \rangle_{F} \\
=\sum_{k,l=1}^{|V|} \sum_{i=1}^{n} [\bold{G}]^i_{kl} \sum_{j=1}^{n}[\gamma]_{i,j}[\bold{G}']^j_{kl}
\end{split}.
\end{equation}
\par The infinite length random walk formulation (see eq. ~\ref{eq:exp_ip}) behaves in a similar way.
Our method though, always compares the same set of edges in the two graphs. \par
Another important difference is the complexity of our method versus the random-walk graph kernel. For an $m\times m$ kernel and $n \times n$ graph, the worst-case complexity for a length $k'$ random walk kernel is $\mathcal{O}(m^2k'n^4)$ and $\mathcal{O}(m^2k'n^2)$ for dense and sparse graphs, respectively \cite{gk_comp}. The GSE, on the other hand, is always $\mathcal{O}\left( m^2n^2\right )$ since the heaviest operation is the Frobenius inner product in order to compute the distance between $\bold{G}$ and $\bold{G}'$. Moreover, once this distance is computed, evaluating the kernel for different values of $\nu$ is $\mathcal{O}(1)$, which combined with the fact that the variance of this kernel is Lipschitz continuous, allows for efficient searching of optimal hyperparameters (see section~\ref{var_region}).

\subsection{Interpretability}\label{LIME}

How could we better understand what the GSE is doing, when it maps points into an infinite-dimensional space? A successful recent development in explaining black-box models is that of Local Interpretable Model-agnostic Explanations
(LIME) \cite{LIME}, where a model is interpreted locally by making slight perturbations in the input and building an interpretable model around the new predictions. We too shall monitor our model's response to changes in the input, but instead of making random perturbations, we will perturb the input in the direction of maximum output change. \par
%It is therefore a Post-hoc Interpretability method \cite{interp}. \\
Given an instance from the dataset $\bold{x}_{1\times N}$, where $N$ is the number of features, and the function that will incorporate the feature connection network $\bold{G}_\bold{x}(\bold{A})$ (e.g. $\bold{G}_\bold{x}(\bold{A}) = \bold{A} \circ \bold{x}^{\top}  \bold{x}$), we will find the direction to which the model is the most sensitive (positive and negative). Unlike optimization, where the goal is to converge as fast as possible, here we are interested in the intermediate steps of the descent. This is because we shall use the set $\mathcal{\bold{G}} = \{\bold{G}_{\bold{x}_1}, \, \bold{G}_{\bold{x}_2}, \,..., \,\bold{G}_{\bold{x}_M}\} $ and the black-box model's predictions $\bold{f} = \{f(\bold{x}_1), \, f(\bold{x}_2), \, ... ,\, f(\bold{x}_M)\}$ to fit our interpretable model $h(\bold{G})\in\mathcal{H}$ (where $\bold{x}_i$ is a variation of the original sample $\bold{x}_0$, and $\mathcal{H}$ represents the space of all possible interpretable functions $h$). This way, we will indirectly unveil the interactions that our model is most sensitive to, and show how these impact the predictions.  
To penalize complex models over simpler ones, we will introduce a function $\Omega(h)$ that measures model complexity. To scale the model complexity term appropriately, we can find a scalar $\theta$ so that the expected value of $\Omega(h)$ is equal to a fraction $\varepsilon$ of the expected value of the loss:
\begin{eqnarray}\label{scaling_complexity_factor}
\mathbf{E}[\theta \Omega(h)] = \varepsilon \mathbf{E}[\mathcal{L}] \leftrightarrow \theta = \frac{\varepsilon \mathbf{E}[\mathcal{L}]}{\mathbf{E}[ \Omega(h)]}
.
\end{eqnarray}
\par Lastly, for highly non-linear models, the larger the input space the more complex the output explanations are likely to be, so we will weigh the sample deviations the same as the original sample $\bold{x}_0$ using the model's own similarity measure $k(\bold{G}_{\bold{x}_i}, \bold{G}_{\bold{x}_0})$. Putting it all together:
\begin{equation}
\xi(\bold{x}_0) = \min_{h\in \mathcal{H}} \mathcal{L}\Big (h, f, k(\bold{G}_{\bold{x}_i}, \bold{G}_{\bold{x}_0})\Big ) + \theta\Omega(h). 
\end{equation}
where $\mathcal{L}\Big (h, f, k(\bold{G}_{\bold{x}_i}, \bold{G}_{\bold{x}_0})\Big )$ is the loss of $h$ when using $\bold{G}_{\bold{x}_i}$ to predict the black-box model output $f(\bold{x}_i)$, weighted by the kernel distance to the original sample $k(\bold{G}_{\bold{x}_i}, \bold{G}_{\bold{x}_0})$.

\subsubsection{Even Descent Sampling Method}
\par In order to adequately cover the most sensitive regions, we need to take steps with equidistant output values. Thus, we developed a novel adaptive method to sample more in steeper regions and less in flatter ones. The intuition is that we would like to approximate the function values in unexplored regions, so that we choose an appropriate sampling step while considering the uncertainty of the approximation. Due to the stochastic nature of the method, it is able to escape local extremes. Consider the value of function $f$ at a point $\bold{x}_0$ and its first order Taylor approximation at an arbitrary point $\bold{x}$:
\begin{equation}
f(\bold{x}) \approx \hat f(\bold{x}) = f(\bold{x}_0) + \nabla_\bold{x} f(\bold{x}_0)(\bold{x}-\bold{x}_0).
\end{equation}
\par The larger the difference $\delta = \bold{x}-\bold{x}_0$, the less likely it is that the approximation error $f(\bold{x})- \hat f(\bold{x})$ is small. Assume we would like to model the random variable $F$, which takes the value of $1$ if the approximation error is small ($\delta=|\hat f(\bold{x}) - f(\bold{x})| \approx 0$), and 0 otherwise.
We will model the probability density function of $F$ as being:
\begin{equation}
p_F(f=1|\delta) \,=\, \lambda e^{-\lambda \delta}.
\end{equation}
\par Consider also the random variable $T$ which takes the value of $1$ if the absolute difference in the output for a point $\bold{x}$ exceeds an arbitrary threshold ($|f(\bold{x})-f(\bold{x}_0)|>\tau$), and $0$ otherwise. Assume there is zero probability this event occurs for sufficiently small steps: $\delta<a(\tau)$, for some value $a(\tau)$. Let us further assume that our confidence that $|f(\bold{x})-f(\bold{x}_0)|>\tau$ increases linearly after the value $\delta=a(\tau)$, until the maximum confidence level is reached at $\delta = b$. After some value $\delta=c$, we decide not to make any further assumptions about this event, so we attribute zero probability from that point on. This can be modeled as:
\begin{equation}
p_T(t=1|\delta) \,=\, \begin{cases}\frac{2}{v}\frac{\delta-a(\tau)}{u} \,,\,\; a(\tau)<\delta\leq b \\
\frac{2}{v}\,,\,\;\,b<\delta\leq c\\
0\,,\,\; \text{otherwise}
\end{cases} ,
\end{equation}
where $v=2c-a(\tau)-b\,,\,u=b-a(\tau)$ and $T=1$, if $|f(\bold{x})-f(\bold{x}_0|>\tau$ and $0$ otherwise.
The distribution of interest is then $p_S=p(f=1 \cap t=1 |\delta)$. To simplify the calculations, we impose the uncertainty about our approximation (expressed by $F$) and the likelihood of a sufficiently large output difference (expressed by $T$) to be independent given $\delta$: $p(f=1 \cap t=1 |\delta) = p(t=1|\delta)p(f=1|\delta)$, and since the goal is to sample steps from this distribution, we will divide it by the normalization constant: $Z=p(f=1 \cap t=1)=\int_{-\infty}^{+\infty}p(f=1 \cap t=1 |\delta)d\delta$. See Figure ~\ref{fig:descentsampdist} for an illustration of the method.
\begin{figure}[htbp!]
	\centering
	\includegraphics[width=0.8\linewidth]{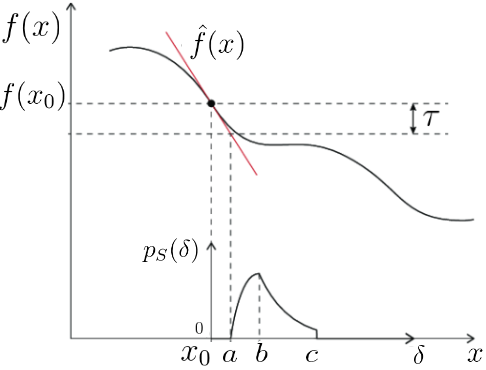}
	\caption{Illustration of the even descent sampling. $\hat f(x)$ approximates the function $f(x)$ and an estimation of how much $\delta=|x-x_0|$ is required to achieve $|f(x)-f(x_0)|\leq\tau$, is computed. Then a sample of $x$ is drawn according to $p_S=p(f=1\cap t=1|\delta)$}
	\label{fig:descentsampdist}
\end{figure}
\par There are a couple of properties that can be manipulated for a successful sampling of the output space:
\subsubsection*{Controlled Termination}
\par To force the algorithm to terminate after a minimum number of samples $M_{min}$ have been sampled, one can decrease the value of $a(\tau)$ with each iteration so that it becomes increasingly more likely that a value of $\delta$ will be picked such that $|f(\bold{x})-f(\bold{x}_0)|<\tau$, terminating the routine. For this purpose, one can compute the estimated threshold value $\tau_0$ that will keep the routine running.
\begin{equation}
|\hat f(\bold{x}) - f(\bold{x}_0)| \geq \tau \Leftrightarrow \sum_{i=1}^N\nabla_\bold{x} f(\bold{x}_0)[i]\delta[i]\geq \tau,
\end{equation}
where $N$ is the number of features. This is an underdetermined equation, but one possible trivial solution is to set:
\begin{equation}\label{eq:tau0}
\delta[i]\geq\frac{\tau}{N'\,\;\nabla_\bold{x} f(\bold{x}_0)[i]}\equiv \tau_0,
\end{equation} %Think about N' (probably incorrect)
where $N'$ is the number of non-zero gradient values,
then let $a(\tau)$ decay with time so that it will reach this limit value after $M_{min}$ iterations:
\begin{equation}
a(\tau)_i = \tau_0\left( 1 + \frac{\theta_a(M_{min}-i)}{M_{min}}\right ).
\end{equation}
\subsubsection*{Escaping Local Extrema}
\par To make it more likely to escape local extrema, one possibility is to set the cut-off value $c$ larger when the norm of $\tau_0$ (eq.~\ref{eq:tau0}) is larger than its expected value, and smaller otherwise:
\begin{equation}
c=b\left(c_l+\frac{\mathbf{E}\left [||\tau_0||_2\right ]-||\tau_0||_2}{\mathbf{E}\left [||\tau_0||_2\right ]+||\tau_0||_2}\right)
\,\,,\,\,c_l\in\,\left]2, +\infty\right[.
\end{equation}
\par This formulation allows jumping out from zones where the gradient is locally small, while taking smaller steps where the gradient is larger than expected.
\subsubsection*{Termination When Too Far from Original Sample} \par Since we are trying to explain the model locally, the sampling should terminate when the algorithm is exploring too far from the original sample. For that purpose, one can set $\lambda$ to increase with increasing distance $d$ to the original sample, pushing the probability density towards the left: $\lambda(d)=e^{-\frac{d}{\sigma^2}}$.\par
Putting all of the above design considerations together, you can find the complete routine in
algorithm ~\ref{even_desc}.

\begin{algorithm}[htbp!]
	\caption{Even Descent algorithm}\label{even_desc}
		\textbf{Input}: $f, \bold{x}_0, \bold{A}$\\
	\textbf{Parameter}: $\tau, \lambda, \theta_a, b, c_l, M_{min}$\\
	\textbf{Output}: $\bold{X}', \bold{f}$ 
	\begin{algorithmic}[1]
		\STATE $i \gets 0$, $f_i \gets f(\bold{x}_0)$, $\bold{f} \gets [f_i]$, converged $\gets$ False
		\STATE $\mathbf{E}[||\tau_0||] = 0$, $\bold{X}' \gets [\bold{x}_0]$
		\WHILE{converged $\not=$ True}
		\STATE $i\gets i+1$
		\STATE $\nabla f\gets$ ComputePartialDers($\bold{x}_0, \bold{A}, f$)
		
		\STATE $\tau_0 \gets \tau/|N' * \nabla f|$ %Think about N' (probably incorrect)
		\STATE $a, b, c \gets$ UpdatepS$\left (i, \theta_a, M_{min}, \mathbf{E}[||\tau_0||], c_l \right )$
		%		\STATE $a \gets \tau_0 \left(1 +\frac{\theta_a(M_{min}-i)}{M_{min}}\right) $\Comment Decaying $a$
	%	\STATE $||\tau_0|| \gets (\sum_j\tau_0^2[j])^\frac{1}{2}$
		%		\STATE $c \gets b  \left(c_l + \frac{\mathbf{E}[||\tau_0||] - ||\tau_0||} { \mathbf{E}[||\tau_0||] + ||\tau_0||}\right)$
		\STATE $\mathbf{E}[||\tau_0||] \gets (\mathbf{E}[||\tau_0||]  (i - 1) + ||\tau_0||) / i$
		\STATE $\delta \gets $ EvenSample($\lambda, a, b, c$)
		\STATE $\bold{x}_i \gets \bold{x}_i \pm \,\, \delta * \nabla f$
		\STATE Append$\left (\bold{f}, f(\bold{x}_i)\right )$, Append$\left( \bold{X}', \bold{x}_i\right)$ 	
		\IF {$|f_i - f_{i-1}| < \tau$}
		\STATE converged $\gets$ True
		\ENDIF
		\ENDWHILE 
		\STATE \textbf{return} $\bold{X}', \bold{f}$
	\end{algorithmic}
\end{algorithm}

\section{Experiments}
\subsection{Materials}
For all our analysis, we used plasma protein levels of patients with suspected coronary artery disease who were diagnosed for the presence of ischaemia \cite{michiel}. A total of 332 protein levels were measured using proximity extension arrays \cite{pea}, and of the 196 patients, 108 were diagnosed with ischaemia. The protein-protein interactions data is available for download at StringDB \cite{string_db}. We implemented the GSE and the random walk kernel in python and used sci-kit learn implementation \cite{scikit-learn} for the other algorithms in the comparison.
\subsection{Ischaemia Classification Performance}\label{benchmark}
	We benchmarked the GSE performance and running time when predicting ischaemia against the random-walk graph kernel, RBF, and random forests. Additionally, in order to test the hypothesis that the protein-interaction information is improving the analysis, we also tested GSE using a constant matrix full of ones as the interaction matrix. For this benchmark, we performed a 10-cycle stratified shuffle cross-validation split
%	 \cite{cv}
on the normalized protein data and recorded the average ROC area under the curve (AUC). To speed up the analysis, we used a training set of 90 pre-selected proteins using univariate feature selection with the F-statistic \cite{f_statistic}. The results are shown in table \ref{performance_table}. 
\begin{table}[htbp!]
\centering
\begin{tabular}{llll}
\hline
\textbf{Method}     &\textbf{AUC std} &\textbf{AUC avg} & \textbf{Run time avg(s)}\\
\hline
\textbf{GSE}&      0.055890 & \textbf{0.814141}   &   7.63\\
\textbf{RWGK} &  0.051704 & 0.808838 &1720\\
\textbf{RF}  &   0.066036 & 0.764141&         17.99\\
\textbf{GSE*} & 0.082309&  0.787879  &        6.59 \\
\textbf{RBF}&  0.095247 & 0.779293  &        1.16\\
\hline
\end{tabular}
\caption{The GSE benchmark against random-walk graph kernel (RWGK), random forests (RF), the GSE with constant interaction matrix (GSE*), and radial basis function (RBF). For all kernels, SVM was used as the learning algorithm.}
\label{performance_table}
\end{table}
The GSE outperformed all the other compared methods, and the fact that the GSE with a constant matrix (GSE*) had a lower performance increases our confidence that the prior interaction knowledge is beneficial for the analysis. The GSE is also considerably faster than the Random-Walk kernel, as expected. To test how both scale increasing feature size, we compared the running time of both for different pre-selected numbers of proteins. The results are depicted in Figure \ref{time_fig}.
\begin{figure}[htbp!]
	\centering
	\includegraphics[width=0.7\linewidth]{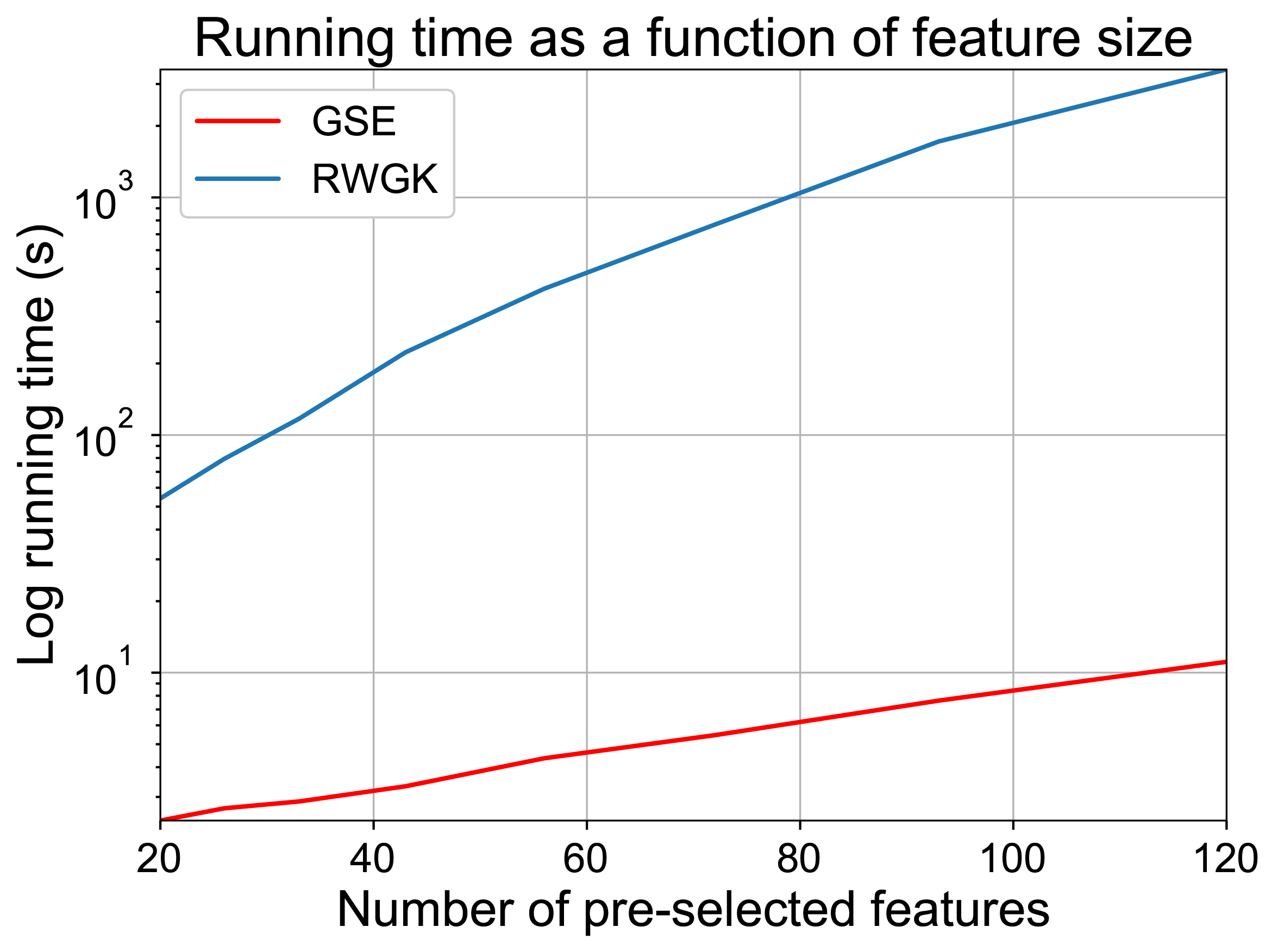}
	\caption{Average running time of the GSE and the Random-walk graph kernel (RWGK), per number of pre-selected features}
	\label{time_fig}
\end{figure}

\subsection{Performance for Different $\nu$ Values}
Recall from section ~\ref{var_region} that a feasible operating region for the $\nu$ values in the GSE kernel was analytically determined. 
We wanted to investigate how the loss function performs within this region, and whether it is possible to draw conclusions regarding the GSE kernel behaviour with respect to the interactions. 
To test this, the $\nu^{\ast}=\max_{\nu} \sigma^2[k(\nu)]$ was found using a gradient descent (ADAM \cite{adam}) on the training set over 20 stratified shuffle splits (same preprocessing as in ~\ref{benchmark}). 
\begin{figure}[htbp!]
	\centering
	\includegraphics[width=0.9\linewidth]{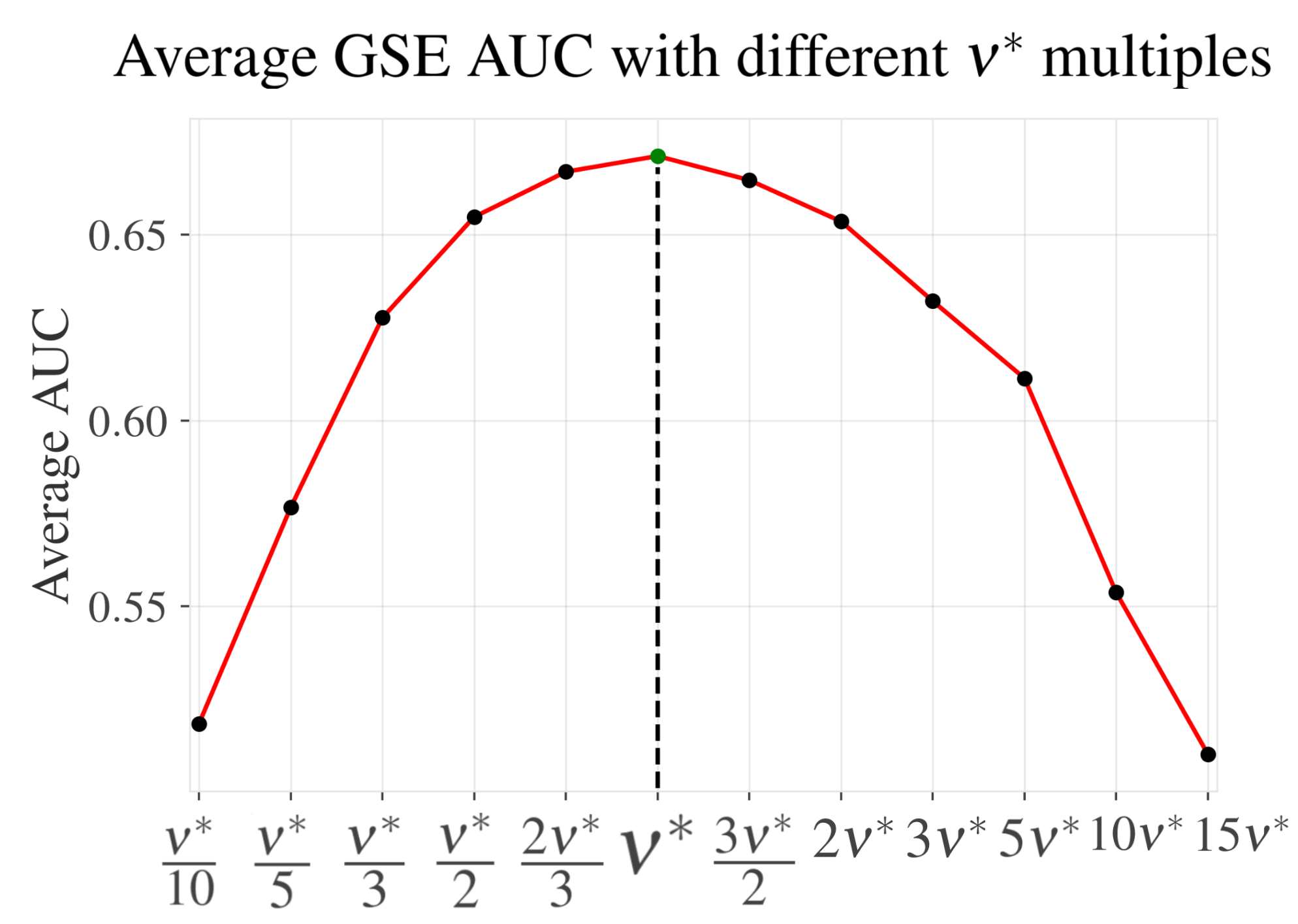}
	\caption{Average ROC AUC on \underline{validation} set using GSE with different $\nu$ values over 20 stratified shuffle splits. Horizontal axis - Multiples of $ \max_{\nu} \sigma^2[k(\nu)]$ here denoted by $\nu^{\ast}$. The AUC as function of the $\nu$ values looks convex and peaks exactly at $\nu^{\ast}$}
	\label{nu_max}
\end{figure}
We then measured the ROC AUC on the validation set using 12 multiples of $\nu^{\ast}$. The results can be seen in Figure ~\ref{nu_max}. 
It is quite interesting that our proxy for measuring kernel expressiveness turns out to be a convex function peaking at $\nu^{\ast}$.

\subsection{Interpretability Test}
To test how interpretable our model's predictions are, first we trained the model on a random subset of our data and used the trained model to predict the rest of the data. Then we employed the method described in section ~\ref{LIME} on a random patient in the test set, using decision trees as the interpretable models $h(\bold{G}) \in \mathcal{H}$, and a linear weighted combination of \textit{max depth} and  \textit{min samples per split} as the complexity penalization term $\Omega(h)$. We then picked the two most important features and made a 3d plot using an interpolation of the prediction space.
The result is depicted in Figure ~\ref{global_lime_fig}.
\begin{figure}[tp!]
	\centering
	\includegraphics[width=1\linewidth]{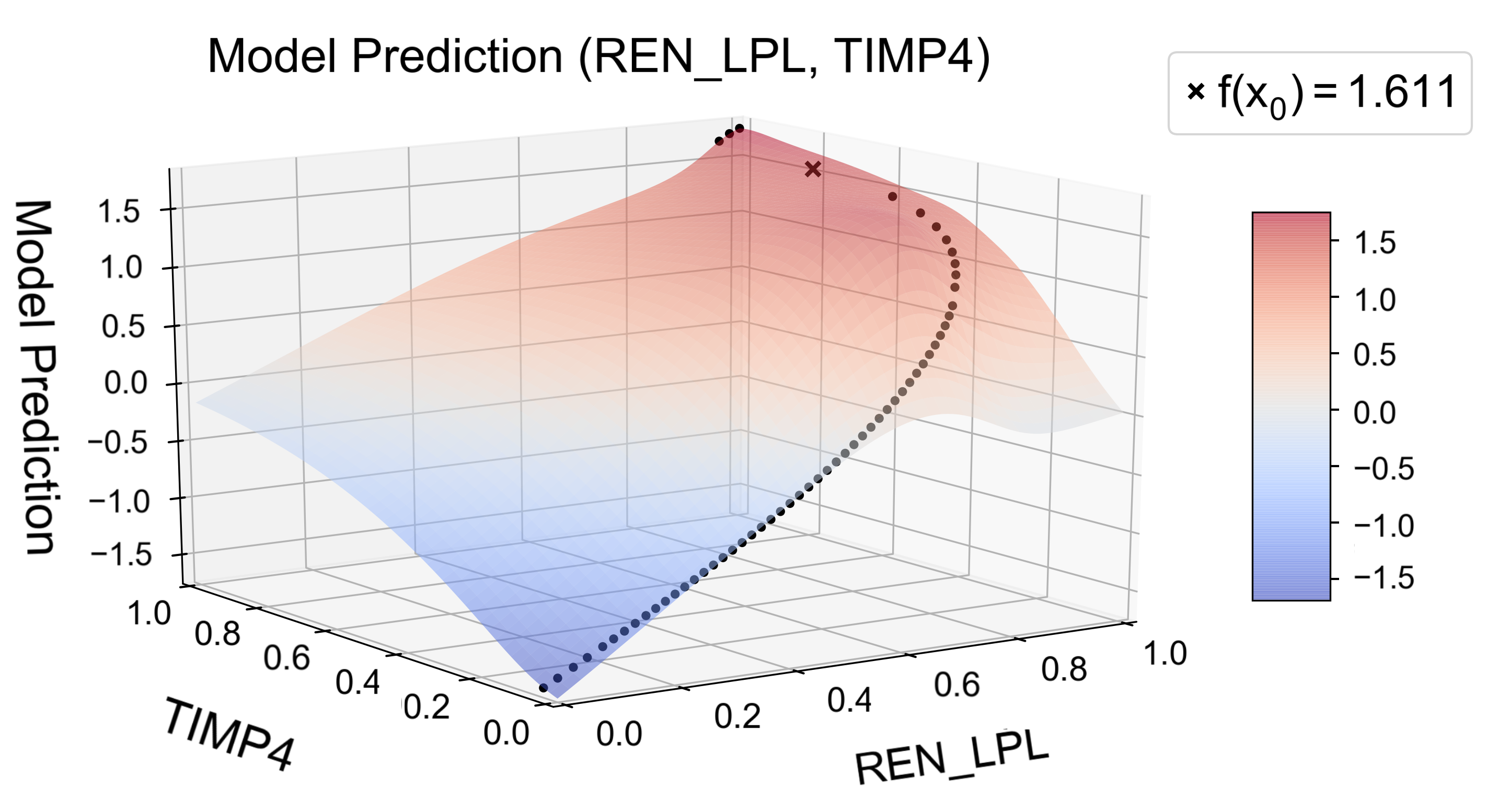}
	\caption{Even Descent Sampling for a random patient in our dataset. This analysis reveals our model "predicts" this patient could be treated by lowering protein "TIMP4" and the interaction between "REN" and "LPL".}
	\label{global_lime_fig}
\end{figure}
\par The Even Descent Sampling tests instances which are approximately equidistant in the output values. For this patient, our model 'predicts' its ischaemia risk could be mitigated by lowering protein TIMP metallopeptidase inhibitor 4 ("TIMP4") and the interaction between
lipoprotein lipase ("LPL") and renin ("REN").

\section{Conclusions}
In this paper, we address the problem of analyzing interconnected systems and leveraging the often-known information about how the components interact. To tackle this task, we developed the \textit{Graph Space Embedding} algorithm and compared it to other established methods using a dataset of proteins and their interactions from a clinical cohort to predict ischaemia. The GSE results outperformed the other algorithms in running time and average AUC. 
Moreover, we presented an optimal regime for the GSE in terms of a feasibility region for its parameters, which vastly decreases the optimization time. Finally, we developed a new technique for interpreting black-box models' decisions, thus making it possible to inspect which features and/or interactions are the most relevant.

\bibliographystyle{named}
\bibliography{ijcai19}

\end{document}